\documentclass[10pt,twocolumn,letterpaper]{article}

\usepackage{cvpr}
\usepackage[utf8]{inputenc}
\usepackage{graphicx}
\usepackage{color}
\usepackage{times}
\usepackage{amsmath,mathtools}
\usepackage{amsthm}	
\usepackage{amssymb}
\usepackage{enumerate}
\usepackage{subfigure}
\usepackage{algorithm} 
\usepackage{algpseudocode}
\usepackage{booktabs}
\usepackage{tabularx}
\usepackage{array}
\usepackage{bm}
\usepackage{url}
\usepackage{bbm}
\usepackage{enumitem}
\usepackage{stackengine}
\stackMath
\usepackage{lipsum}

\usepackage[algo2e,inoutnumbered,algoruled,vlined]{algorithm2e} %
\usepackage{xcolor}

\usepackage[pagebackref=true,breaklinks=true,letterpaper=true,colorlinks,bookmarks=false]{hyperref}

\usepackage{tabularx}

\newcommand{\PBS}[1]{\let\temp=\\#1\let\\=\temp}
\newcommand{\RBS}{\let\\=\tabularnewline}
\newcommand{\F}{{{\mathcal F}}}

\DeclareMathSymbol{@}{\mathord}{letters}{"3B}

\cvprfinalcopy 

\def\cvprPaperID{283} 

\ifcvprfinal\fi

\newcommand{\x}{\mathbf{x}}

\newcommand{\ab}{\mathbf{a}}

\newcommand{\C}{\mathbf{C}}

\newcommand{\QH}{\mathbb{H}}

\newcommand{\one}{\mathbf{1}}

\newcommand{\X}{\mathbf{X}}

\newcommand{\q}{\mathbf{q}}

\newcommand{\qc}{\mathbf{\bar{q}}}

\newcommand{\VM}{{\cal V}} 
\newcommand{\Edges}{{\cal E}} 

\newcommand{\mur}{\bm{\mu}}
\newcommand{\muc}{\bm{\nu}}
\newcommand{\meas}{\bm{\mu}}

\newcommand{\vb}{\mathbf{v}}

\newcommand{\y}{\mathbf{y}}

\newcommand{\R}{\mathbb{R}}
\newcommand{\Rot}{\mathbf{R}}

\newcommand{\Pb}{\mathbf{P}}
\newcommand{\bp}{\mathbf{p}}
\newcommand{\Exp}{\text{\emph{Exp}}}
\newcommand{\Log}{\text{\emph{Log}}}
\newcommand{\T}{\mathcal{T}}

\newcommand{\etab}{\bm{\eta}}

\newcommand*\diff{\mathop{}\!\mathrm{d}}

\DeclarePairedDelimiterX{\infdivx}[2]{(}{)}{%
  #1\;\delimsize\|\;#2%
}

\newtheorem{thm}{Theorem}

\newtheorem{lemma}{Lemma}
\newtheorem{prop}{Proposition}
\newtheorem{dfn}{Definition}

\DeclareMathOperator*{\st}{s.t.}

\DeclareMathOperator{\sign}{sign}

\usepackage{cleveref}

\newtheorem{assumption}{\textbf{H}\hspace{-3pt}}
\Crefname{assumption}{\textbf{H}\hspace{-3pt}}{\textbf{H}\hspace{-3pt}}
\crefname{algorithm}{\text{Alg.}}{\text{Alg.}}
\crefname{assumption}{\textbf{H}}{\textbf{H}}
\crefname{equation}{\text{Eq}}{\text{Eq}}
\crefname{definition}{\text{Dfn.}}{\text{Dfn.}}
\crefname{lemma}{\text{Lemma}}{\text{Lemma}}
\crefname{dfn}{\text{Dfn.}}{\text{Dfn.}}
\crefname{thm}{\text{Thm.}}{\text{Thm.}}
\crefname{tab}{\text{Tab.}}{\text{Tab.}}
\crefname{fig}{\text{Fig.}}{\text{Fig.}}
\crefname{table}{\text{Tab.}}{\text{Tab.}}
\crefname{figure}{\text{Fig.}}{\text{Fig.}}
\crefname{section}{\text{Sec.}}{\text{Sec.}}

\newcommand{\insertimageC}[5]{ 
\begin{figure}[#5]
\centering
\includegraphics[width=#1\linewidth, clip=true]{figures/#2}
\caption{#3}
\label{#4}
\end{figure}
}

\newcommand{\insertimageStar}[5]{ 
\begin{figure*}[#5]
\centering
\includegraphics[width=#1\linewidth, clip=true]{figures/#2}
\caption{#3}
\label{#4}
\end{figure*}
}

\begin{document}

\title{\vspace{-5mm}Synchronizing Probability Measures on Rotations via Optimal Transport}

\author{
  Tolga Birdal\textsuperscript{ 1}
  \qquad
  Michael Arbel\textsuperscript{ 2}
  \qquad Umut \c{S}im\c{s}ekli\textsuperscript{ 3, 4}
  \qquad Leonidas Guibas \textsuperscript{1}
  \\
  \textsuperscript{1 }{Department of Computer Science, Stanford University, USA}\\
  \textsuperscript{2 }{Gatsby Computational Neuroscience Unit, University College London, UK}\\
  \textsuperscript{3 }{LTCI, T\'{e}l\'{e}com Paris, Institut Polytechnique de Paris, Paris, France}\\
  \textsuperscript{4 }{Department of Statistics, University of Oxford, Oxford, UK}
}

\maketitle

\begin{abstract}
We introduce a new paradigm, `measure synchronization', for synchronizing graphs with measure-valued edges. We formulate this problem as maximization of the cycle-consistency in the space of probability measures over relative rotations. In particular, we aim at estimating marginal distributions of absolute orientations by synchronizing the `conditional' ones, which are defined on the Riemannian manifold of quaternions.
Such graph optimization on distributions-on-manifolds enables a natural treatment of multimodal hypotheses, ambiguities and uncertainties arising in many computer vision applications such as SLAM, SfM, and object pose estimation. We first formally define the problem as a generalization of the classical rotation graph synchronization, where in our case the vertices denote probability measures over rotations. We then measure the quality of the synchronization by using Sinkhorn divergences, which reduces to other popular metrics such as Wasserstein distance or the maximum mean discrepancy as limit cases. We propose a nonparametric Riemannian particle optimization approach to solve the problem. Even though the problem is non-convex, by drawing a connection to the recently proposed sparse optimization methods, we show that the proposed algorithm converges to the global optimum in a special case of the problem under certain conditions. Our qualitative and quantitative experiments show the validity of our approach and we bring in new perspectives to the study of synchronization.
\end{abstract}

\vspace{-3mm}\section{Introduction}
\label{sec:intro}
\textit{Synchronization}~\cite{simons1990overview,giridhar2006distributed,singer2011angular}, is the art of consistently recovering absolute quantities from a collection of ratios or pairwise relationships. An equivalent definition describes the problem as simultaneously upgrading from pairwise local information onto the global one: an \textit{averaging}~\cite{govindu2014averaging,hartley2013rotation,Arrigoni2019}. At this point, synchronization is a fundamental piece of most multi-view reconstruction / multi-shape analysis pipelines~\cite{salas2013slam++,cadena2016past,carlone2015initialization} as it heavy lifts the global constraint satisfaction while respecting the geometry of the parameters. In a wide variety of applications, these parameters are composed of the elements of a rotation expressed either as a point on a Riemannian manifold (forms a Lie group) or its tangent space (Lie algebra)~\cite{govindu2004lie,birdal2018bayesian}. 
\insertimageC{1}{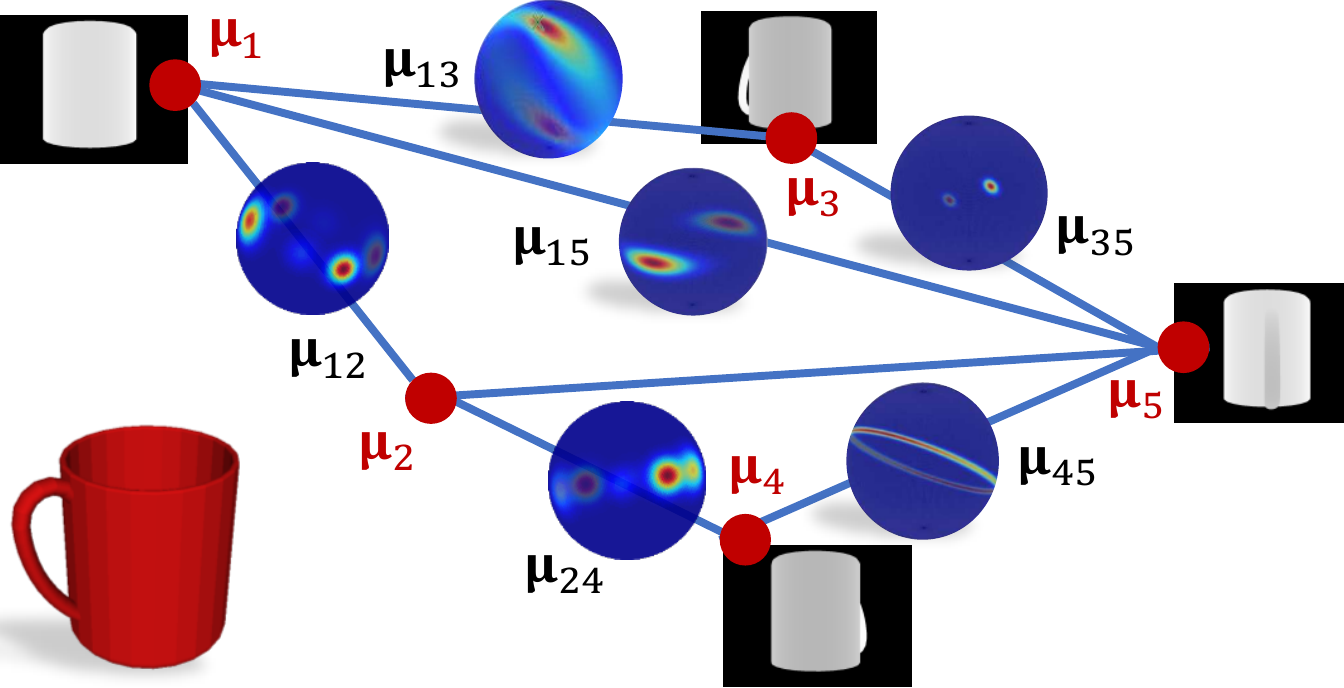}{We coin \textbf{measure synchronization} to refer to the problem of recovering distributions over absolute rotations $\{\meas_i\}_i$ (shown in red) from a given set of potential relative rotations $\{\meas_{ij}\}_{(i,j)}$ per pair that are estimated in an independent fashion.\vspace{-3mm}
}{fig:measuresync}{t!}

Unfortunately, for our highly complex environments, symmetries are inevitable and they may render it impossible to hypothesize an ideal single candidate to explain the configuration of capture, be it pairwise or absolute. In those cases, there is no single best solution to the 3D perception problem. One might then speak of a set of plausible solutions that are only collectively useful in characterizing global positioning of the cameras / objects~\cite{kendall2015posenet,manhardt2019explaining,birdal2018bayesian}. Put more concretely, imagine a 3D sensing device capturing a rotating untextured, cylindrical coffee mug with a handle (\cf~\cref{fig:measuresync}). The mug's image is, by construction, identical under the rotations around its symmetry axis whenever the handle is occluded. Any relative or absolute measurement is then subject to an \textit{ambiguity} which prohibits the estimation of a unique best alignment. Whenever the handle appears, mug's pose can be determined uniquely, leading to a single mode. Such phenomena are ubiquitous in real scenarios and have recently lead to a paradigm shift in computer vision towards reporting the estimates in terms of either empirical or continuous multimodal probability distributions on the parameter spaces~\cite{kendall2016modelling,manhardt2019explaining,birdal2018bayesian,birdal2019birkhoff}, rather than single entities as done in the conventional methods~\cite{nister2004efficient,melekhov2017relative,hartley2003multiple}.

In the aforementioned situations, to the best of our knowledge, all of the existing synchronization algorithms~\cite{govindu2004lie,rosen2016se,Eriksson2018,govindu2014averaging,sun2019,arrigoni2016spectral} are doomed to failure and a new notion of synchronization is required that rather takes into account the distributions that are defined on the parameter space. In other words, what are to be synchronized are now the probability measures and not single data points\footnote{We note that our framework extends the classical one in the sense that single data points can be seen as degenerate probability distributions represented by Dirac-delta functions.}. 

In this paper, considering the {Sinkhorn divergence}, which has proven successful in many application domains \cite{cuturi2013sinkhorn,feydy2019,genevay2018sample}, we introduce the concept of \textit{measure synchronization}. After formally defining the graph synchronization problem over probability measures by using the Sinkhorn divergence, we propose a novel algorithm to solve this problem, namely the Riemannian particle gradient descent (RPGD), which is based on the recently developed nonparametric `particle-based' methods \cite{liu2017stein,arbel2019maximum,liutkus2019sliced,kolouri2019generalized} and Riemannian optimization methods~\cite{absil2009optimization}. 
In particular, we parameterize rotations by quaternions and use the Sinkhorn divergence~\cite{feydy2019} between the estimated joint rotation distribution and the input empirical relative rotation distribution as a measure of consistency. 

We further investigate the special case of the proposed problem, where the loss function becomes the \emph{maximum mean discrepancy} (MMD) combined with a simple regularizer. By using this special structure and imposing some structure on the weights of the measures to be estimated, we draw a link between our approach and the recently proposed sparse optimization of~\cite{chizat2019sparse}. This link allows us to establish a theoretical guarantee showing that RPGD converges to the global optimum under certain conditions. 

Our synthetic and real experiments demonstrate that both of our novel, non-parametric algorithms are able to faithfully discover the probability distribution for each of the absolute cameras with varying modes, given the multimodal rotation distribution for each edge in the graph. This is the first time the two realms of synchronization and optimal transport are united and we hope to foster further research leading to enhanced formulations and algorithms. To this end, we will make our code as well as the datasets publicly available.
In a nutshell, our contributions are:
\begin{enumerate}[noitemsep]
    \item We introduce the problem of measure synchronization that aims for synchronization of probability distributions defined on Riemannian manifolds. 
    \item We address the particular case of multiple rotation averaging and solve this new synchronization problem via minimizing a collection of regularized optimal transport costs over the camera pose graph.
    \item We propose the first two solutions resulting in two algorithms: \textit{Sinkhorn-Sync} and \textit{MMD-Sync}. These algorithms differ both by the topology they induce and by their descent regimes.
    \item We show several use-cases of our approach under synthetic and real settings, demonstrating the applicability to a general set of challenges.
\end{enumerate}
We release our source code under: \url{https://synchinvision.github.io/probsync/}.

\section{Prior Art}
\label{sec:related}

\paragraph{Synchronization.}
The emergence of the term \textit{synchronization} dates back to Einstein \& Poincar\'e and the principles of relativity~\cite{galison2003einstein}. It was first used to recover clocks~\cite{simons1990overview}. Since then the ideas have been extended to many domains other than space and time, leaping to communication systems~\cite{golomb1963synchronization,giridhar2006distributed}, finally arriving at the graphs in computer vision with Singer's angular synchronization~\cite{singer2011angular}. Various aspects of the problem have been vastly studied: different group structures~\cite{govindu2004lie,birdal2019birkhoff,arrigoni2017synchronization,Arrigoni2019,huang2019tensor}, closed form solutions~\cite{arrigoni2016spectral,arrigoni2017synchronization,Arrigoni2019}, robustness~\cite{chatterjee2017robust}, certifiability~\cite{rosen2019se}, global optimality~\cite{briales2017cartan}, learning-to-synchronize~\cite{huang2019learning} and uncertainty quantification (UQ)~\cite{tron2014statistical,birdal2018bayesian}.  
The latter recent work of Birdal~\etal~\cite{birdal2018bayesian} and the $K$-best synchronization~\cite{sun2019} are probably what relates to us the most. In the former a probabilistic framework to estimate an empirical absolute distribution and thereby the uncertainty per camera is proposed. The latter uses a deterministic approach to yield $K$ hypotheses with the hope that ambiguities are captured. Yet, none of the works in the synchronization field including~\cite{birdal2018bayesian,birdal2019birkhoff,sun2019} are able to handle multiple relative hypotheses and thus are prone to failure when ambiguities are present in the input. This is what we overcome in this work by proposing to align the observed and computed relative probability measures: \emph{measure synchronization}.
\vspace{-3mm}
\paragraph{Optimal transport and MMD.}
Comparing probability distributions has even a longer history and is a very well studied subject~\cite{mahalanobis1936generalized,kullback1951information}. 
Taking into account the geometry of the underlying space, optimal transport (OT)~\cite{villani2008optimal}, proposed by Monge in 1781~\cite{monge1781memoire} and advanced by Kantorovich~\cite{kantorovich2006translocation} seeks to find the minimum amount of effort needed to morph one probability measure into another one. This geometric view of comparing empirical probability distributions metrizes the convergence in law better capturing the topology of measures~\cite{feydy2019}. 
Thanks to the efficient algorithms developed recently~\cite{cuturi2013sinkhorn,rabin2011wasserstein,solomon2015convolutional}, OT has begun to find use-cases in various sub-fields related to computer vision such as deep generative modeling~\cite{arjovsky2017wasserstein}, non-parametric modeling~\cite{liutkus2019sliced}, domain adaptation~\cite{courty2016optimal}, 3D shape understanding~\cite{solomon2014earth}, graph matching~\cite{xu19b}, topological analysis~\cite{lacombe2018large}, clustering~\cite{mi2018variational}, style transfer~\cite{kolkin2019style} or image retrieval~\cite{rubner2000earth} to name a few.

Only recently Feydy~\etal~\cite{feydy2019} formulated a new geometric Sinkhorn~\cite{sinkhorn1967concerning,sinkhorn1964relationship} divergence that can relate OT distances, in particular the Sinkhorn distance~\cite{cuturi2013sinkhorn} to \emph{maximum mean discrepancy} (MMD)~\cite{gretton2007kernel}, a cheaper to compute frequentist norm with a smaller sample complexity. MMD has been successfully used in domain adaptation~\cite{tzeng2014deep}, in GANs~\cite{li2017mmd, arbel2018gradient}, as a critic~\cite{sutherland2016generative} as well as to build auto-encoders~\cite{tolstikhin2018wasserstein}. Furthermore, Arbel~\etal~\cite{arbel2019maximum} have established the Wasserstein gradient flow of the MMD and demonstrated global convergence properties based on a novel noise injection scheme. Jointly, these advancements allow us to tackle the proposed measure synchronization problem using MMDs with global optimality guarantees.

To the best of our knowledge, no prior work addresses the measure synchronization problem we consider here. The closest work is by Solomon~\etal~\cite{solomon2014wasserstein} where a Wasserstein metric is minimized so as to perform a graph transduction: the labels are propagated from the \emph{absolute} entities along the edges (relative), not the other way around.
\section{Background}
\label{sec:OT}
We now define and review the necessary concepts required for the grasp of our method. 
\subsection{Rotations and Synchronization}
\begin{dfn}[Rotation Graph]
We consider a finite simple directed graph $G = (\VM, \Edges)$, where vertices correspond to reference frames and edges to the available relative measurements. Both vertices and edges are labeled with absolute and relative orientations, respectively. Each absolute frame is described by a rotation matrix $\{\Rot_i\in SO(3)\}_{i=1}^n$. Each relative pose is expressed as the transformation between frames $i$ and $j$, $\Rot_{ij}$, where $(i,j) \in \Edges \subset [n] \times [n]$.
$G$ is undirected such that if $(i, j) \in \Edges$, then $(j, i) \in \Edges$ and $\Rot_{ij}=\Rot_{ji}^{-1}$. 
\end{dfn}

\begin{dfn}[Rotation Synchronization]
Synchronizing pose graphs, also known as multiple rotation averaging, involves the task of relating different coordinate frames by satisfying the compatibility constraint~\cite{govindu2001combining,hartley2013rotation}:
\begin{equation}
\label{eq:problem}
\Rot_{ij} \approx \Rot_{j}\Rot_{i}^{-1},\, \forall i \neq j
\end{equation}
\end{dfn}
\begin{dfn}[Quaternion]
In the rest of the paper, we will represent a 3D rotation by a \emph{quaternion} $\q$, an element of Hamilton algebra $\QH$, extending the complex numbers with three imaginary units $\textbf{i}$, $\textbf{j}$, $\textbf{k}$ in the form
$\q
		= q_1 \textbf{1} + q_2 \textbf{i} + q_3 \textbf{j} + q_4 \textbf{k}
    = \left(q_1, q_2, q_3, q_4\right)^{\text{T}}$,
with $\left(q_1, q_2, q_3, q_4\right)^{\text{T}} \in \mathbb{R}^4$ and
$\textbf{i}^2 = \textbf{j}^2 = \textbf{k}^2 = \textbf{i}\textbf{j}\textbf{k} = - \textbf{1}$. We also write $\q := \left[a, \textbf{v}\right]$ with the scalar part $a = q_1 \in \mathbb{R}$ and the vector part $\textbf{v} = \left(q_2, q_3, q_ 4\right)^{\text{T}} \in \mathbb{R}^3$.
The conjugate $\bar{\q}$ of the quaternion $\q$ is given by $
	\qc := q_1 - q_2 \textbf{i} - q_3 \textbf{j} - q_4 \textbf{k}$.
A versor or \emph{unit quaternion} $\q \in \mathbb{H}_1$ with $1 \stackrel{\text{!}}{=} \left\|\q\right\|
	:= \q \cdot \qc$ and $\q^{-1}= \qc$,
gives a compact and numerically stable parametrization to represent orientation of objects in $\mathbb{S}^3$, avoiding gimbal lock and singularities~\cite{busam2016_iccvw}.  Identifying antipodal points $\q$ and $-\q$ with the same element, the unit quaternions form a double covering group of $SO\left(3\right)$.
The non-commutative multiplication of two quaternions $\mathbf{p}:=[p_1, \mathbf{v}_p]$ and $\mathbf{r}:=[r_1, \mathbf{v}_r]$ is defined to be $\textbf{p}\otimes\textbf{r} =
[{p}_1{r}_1-\mathbf{v}_p\cdot \mathbf{v}_r,\,{p}_1\mathbf{v}_r+{r}_1 \mathbf{v}_p+\mathbf{v}_p \times \mathbf{v}_r]$.
For simplicity we use $\textbf{p}\otimes\textbf{r}:=\textbf{p}\cdot\textbf{r}:=\textbf{p}\textbf{r}$.
\end{dfn}
\begin{dfn}[Relative Quaternion]
The relative rotation between two frames can also be conveniently written in terms of quaternions: $\q_{ij} = \q_i \otimes \bar{\q}_j$.
Note that for quaternions we will later omit the symbol $\otimes$. 
\end{dfn}

\paragraph{Riemannian geometry of quaternions}
We will now briefly explain the Lie group structure of the quaternions $\QH$ that are essential for deriving the update equations on the manifold. Our convention follows~\cite{angulo2014riemannian}.
\begin{dfn}[Exponential Map]
The \emph{exponential map} $\Exp_{\q}(\cdot)$ maps any vector in $\T_{\q}\QH$, the tangent space of $\QH$, onto $\QH$: 
\begin{align}
    \Exp_{\q}(\etab)=\q\exp(\etab) = \q\Bigg(  e^w \big( \cos(\theta), \vb \frac{\sin(\theta)}{\theta}\big)\Bigg),
\end{align}
where $\etab=(w, \vb)\in \T_{\q}\QH$ and $\theta = \| \vb \|$.
\end{dfn}
\begin{dfn}[Logarithmic Map]
The inverse of $\emph{exp-map}$, $\Log_{\q}(\bp) : \QH \mapsto \T_{\q}\QH$ is called the $\emph{log-map}$ and defined as:
\begin{align}
    \Log_{\q}(\bp)=\log(\q^{-1}\bp) = \Big(0, \frac{\vb}{\|\vb\|} \arccos(w) \Big),
\end{align}
this time with a slight abuse of notation $\q^{-1}\bp=(w, \vb)$. 
\end{dfn}
\begin{dfn}[Riemannian Distance]
\label{def:riemann_dist}
Let us rephrase the Riemannian distance between two unit quaternions using the logarithmic map, whose norm is the length of the shortest geodesic path. Respecting the antipodality:
\begin{align}
    d(\q_1,\q_2) = \begin{cases}
    \| \Log_{\q_1}(\q_2) \| \,\,\,\,=  \arccos(w), & w\geq 0 \\
    \| \Log_{\q_1}(-\q_2) \| =\arccos(-w), & w<0
    \end{cases}
    \nonumber
\end{align}
where $\q_1^{-1}\q_2 = (w,\vb)$.
\end{dfn}


\subsection{Optimal Transport}

\begin{dfn}[Discrete Probability Measure]
We define a discrete probability measure $\mur$ on the simplex $\Sigma_d\triangleq\{\x \in \R_+^n : \x^\top\one_n=1\}$ with weights $\ab=\{a_i\}$ and locations $\{x_i\}, i=1\dots n$ as:
\begin{align}
    \mur = \sum\nolimits_{i=1}^n a_i \delta_{x_i}, \quad a_i\geq 0 \,\,\wedge\,\, \sum\nolimits_{i=1}^n a_i = 1 
    \label{eqn:discrete_meas}
\end{align}
where $\delta_{x}$ is a Dirac delta function at $x$. 

We further define the product of two measures $\mur = \sum_{i=1}^{n_{\mur}} a^{\mur}_i \delta_{x^{\mur}_i}$ and $\muc = \sum_{i=1}^{n_{\muc}} a^{\muc}_i \delta_{x^{\muc}_i}$, as follows:
\begin{align}
\mur \otimes \muc \triangleq \sum\nolimits_{i=1}^{n_{\mur}} \sum\nolimits_{j=1}^{n_{\muc}} a^{\mur}_i a^{\muc}_j \delta_{x^{\mur}_i} \delta_{x^{\muc}_j}.
\end{align}
\end{dfn}

\begin{dfn}[Transportation Polytope]
For two probability measures $\mur$ and $\muc$ in the simplex, let $U(\mur,\muc)$ denote the polyhedral set of $n_{\mur}\times n_{\muc}$ matrices:
\begin{align*}
    U(\mur,\muc) = \{\Pb\in\R^{n_{\mur}\times n_{\muc}}_{+} \,:\, \Pb\one_{n_{\muc}} = \mur \,\wedge\, \Pb^\top \one_{n_{\mur}} = \muc \}
\end{align*}
$\one_d$ is a $d$-dimensional ones-vector.
\end{dfn}

\begin{dfn}[Kantorovich’s Optimal Transport Problem]
Let $\C$ be a $n_{\mur}\times n_{\muc}$ cost matrix that is constructed from the \emph{ground cost} function $c(x_i^{\mur}, x_j^{\muc}) $. Kantorovich's problem seeks to find the transport plan optimally mapping $\mur$ to $\muc$:
\begin{align}
    \mathcal{W}_c(\mur, \muc) = d_{c}(\mur, \muc) = \min\nolimits_{\Pb\in U(\mur,\muc)} \langle \Pb, \C \rangle 
\end{align}
with $\langle\cdot\rangle$ denoting the Frobenius dot-product. Note that this expression is also known as the \textit{Wasserstein distance} (WD).
\end{dfn}

\begin{dfn}[Maximum Mean Discrepancy]
\label{dfn:mmd}
Given a characteristic positive semi-definite kernel $k$, the maximum mean discrepancy (MMD)(\cite{gretton2007kernel}) defines a distance between probability distributions. In particular, for discrete measures $\mur$ and $\muc$ and when  $k$ is induced from a ground cost function $c$: $k(x,y) = \exp(-c(x,y))$, the squared MMD is given by:
\begin{align}
    \mathcal{MMD}_k^2(\mur,\muc) =  \bar{K}_{\mur,\mur} + \bar{K}_{\muc,\muc}  - 2\bar{K}_{\mur,\muc} 
\end{align}
where $\bar{K}_{\mur,\muc} $ is an expectation of $k$ under $\mur\otimes\muc$:
\begin{align}
	\bar{K}_{\mur,\muc} = \sum_{i=1}^{n_{\mur} }\sum_{j=1}^{n_{\mur}} a_i^{\mur} a_j^{\muc} \exp{(-c(x_i^{\mur},x_j^{\muc}))}
\end{align}
\end{dfn}
\begin{dfn}[Sinkhorn Distances]
\label{eq:sinkhorn}
Cuturi~\cite{cuturi2013sinkhorn} introduced an alternative convex set consisting of joint probabilities with small enough mutual information:
\begin{align}
    U_{\alpha}(\mur,\muc) = \{\Pb\in U(\mur,\muc) \,:\, \emph{KL}\infdivx{\Pb} {\mur\muc^\top}\leq \alpha \}
\end{align}
where $\emph{KL}(\cdot)$ refers to the Kullback Leibler divergence~\cite{kullback1951information} between $\mur$ and $\muc$. This restricted polytope leads to a tractable distance between discrete probability measures under the cost matrix $\C$ constructed from the function $c$:
\begin{align}
    d_{c,\alpha}(\mur, \muc) = \min\nolimits_{\Pb \in U_{\alpha}(\mur,\muc)} \langle \Pb, \C \rangle.
\end{align}
$d_{c,\alpha}$ can be computed by a modification of simple matrix scaling algorithms, such as Sinkhorn's algorithm~\cite{sinkhorn1967diagonal}.
\end{dfn}
The discrepancy $d_{c,\alpha}$ suffers from an entropic bias, that is corrected by \cite{feydy2019,genevay2017learning} giving rise to a new family of divergences as we define below:
\begin{dfn}[Sinkhorn Divergences]\label{eq:sinkdiv}
$\mathcal{S}_{c,\alpha}(\mur, \muc)$ is defined to be the Sinkhorn divergence:
\begin{align}
\mathcal{S}_{c,\alpha}(\mur, \muc) \triangleq 2d_{c,\alpha}(\mur, \muc)-d_{c,\alpha}(\mur, \mur)-d_{c,\alpha}(\muc, \muc).\nonumber
\end{align}
Unlike~\cref{eq:sinkhorn}, this satisfies $\mathcal{S}_{c,\alpha}(\mur, \mur)=\mathcal{S}_{c,\alpha}(\muc, \muc)=0$ and interpolates between the WD and MMD with kernel $k=-c$:
\begin{align*}
\mathcal{W}_c(\mur, \muc) \xleftarrow{\alpha \rightarrow 0} \mathcal{S}_{c,\alpha}(\mur, \muc) \xrightarrow{\alpha \rightarrow \infty} \frac{1}{2} \mathcal{MMD}^2_{-c}(\mur,\muc).
\end{align*}
\end{dfn}

\section{Measure Synchronization}
In this section, we will introduce the measure synchronization problem. In our problem setting, for each camera pair $(i,j) \in \Edges$, we will assume that we \emph{observe} a discrete probability measure $\meas_{ij}$, denoting the distribution of the relative poses between the cameras $i$ and $j$. More formally, such \emph{relative pose measures} will consist of a collection of weighted Dirac masses (cf.\ \eqref{eqn:discrete_meas}), given as follows:
\begin{align}
\meas_{ij} \triangleq \sum\nolimits_{k=1}^{K_{ij}} w^{(k)}_{ij} \delta_{\q^{(k)}_{ij}}, \label{eqn:rel_meas}
\end{align}
where $\{\q_{ij}^{(k)}\}_{k=1}^{K_{ij}}$ denote all the possible \emph{relative pose masses} and $K_{ij}$ denotes the total number of potential relative poses for a given camera pair $(i,j)$, whereas $w_{ij}^{(k)}\in \mathbb{R}_+$ denotes the weight of each point mass $\q_{ij}^{(k)}$. We will assume that these measures can be estimated (in a somewhat noisy manner) by using an external algorithm, such as~\cite{melekhov2017relative,nister2004efficient,deng20193d}. 

Given all the `noisy' relative pose measures $\{\meas_{ij}\}_{(i,j)\in \Edges}$, similar to the conventional pose synchronization problem, our goal becomes estimating the \emph{absolute pose measures} $\meas_i$, which are a-priori unknown and denote the probability distribution of the absolute pose corresponding to camera $i$. Similar to the relative pose measures, we define the absolute pose measures as follows:
\begin{align}
\meas_{i} \triangleq \sum\nolimits_{k=1}^{K_{i}} w^{(k)}_{i} \delta_{\q^{(k)}_{i}},
\end{align}
where $\{\q^{(k)}_{i},  w^{(k)}_{i}\}_{k=1}^{K_i}$ are the point mass-weight pairs in $\meas_i$ and $K_i$ denotes the number of points in $\meas_i$. Semantically,  $\q^{(k)}_{i}$ and $w^{(k)}_{i}$ will respectively denote all the `plausible' poses and their weights for camera $i$. Our ultimate goal is to estimate the absolute pose measures such that they will produce `compatible' relative measures in the sense of~\cref{eq:problem}. In other words, we aim at estimating the pairs $\{\{\q^{(k)}_{i},  w^{(k)}_{i}\}_{k=1}^{K_i}\}_{i=1}^n$, that achieve a good graph consistency at a probability distribution level, a notion which will be precised in the sequel. 

Let us introduce some notation that will be useful for defining the ultimate problem. We denote by $\meas$ a coupling between $(\q_1, \dots,\q_n)$
which is of the form:
\begin{align}\label{eq:general_join}
	\meas = \sum_{k_1=1}^{K_1}... \sum_{k_n=1}^{K_n} v^{(k_1,...,k_n)}\prod_{i=1}^n \delta_{q_i^{k_i}}.
\end{align}
where $v^{(k_1,...,k_n)}$ are non-negative weights that sum to $1$. $\meas$ represents the joint distribution over $(\q_1,...,\q_n)$ and recovers each marginal distributions $\meas_i$ after marginalizing over all the other ones. 
Although one could, in principle, learn the general form of \cref{eq:general_join}, we  will be interested in two extreme cases for $\meas$, a \textit{high entropy} case (HE)  and a \textit{low entropy} case (LE) for which \cref{eq:general_join} simplifies. In the (HE) case, we assume that $\meas$ is fully factorized, i.e:
\begin{align}
	v^{(k_1,...,k_n)} = \prod\nolimits_{i=1}^n w_{i}^{(k_i)}.
\end{align}
This case occurs when all information about the pairings between the absolute poses measures $\meas_i$ is lost.
In the (LE) case, $v^{(k_1,...,k_n)} = 0$ whenever $k_1,...k_n$ are not all equal to each other. This implies in particular that all marginals $\meas_i$ have the same number of point masses $K_1=...=K_n=K$ and share the same weights: 
\begin{align}
	w_i^{(k)}= v^{(k,...,k)} \triangleq v^{(k)}.
\end{align}
This setting is `simpler' in the sense that it carries more information about the absolute poses. We will illustrate that both of these settings will be useful in different applications.

We propose to compute a joint distribution $\meas$ in either case that is consistent with the relative pose measures. 
As a first step towards this, we further introduce the `composition functions' $g_{ij}$, which are simply given by:
\begin{align}
	g_{ij}(\q_1,...,\q_n) \triangleq \q_i\overline{\q}_j.
\end{align}
We use $g_{ij}$ to construct `ideal relative pose measures' from the joint absolute measure $\meas$ by a push-forward operation, again denoted by $g_{ij}(\meas)$:
\begin{align}\label{eqn:comp_fn}
 g_{ij}(\meas) 
&\triangleq \sum\nolimits_{k_1=1}^{K_1}... \sum\nolimits_{k_n=1}^{K_n}  v^{(k_1,...,k_n)}\delta_{\q^{(k_i)}_i \overline{\q^{(k_j)}_j}}. 
\end{align}
In both cases (HE) and (LE), $g_{ij}$ admits a simpler expression which is provided in the supplementary document.

We are now ready to formally the define the measure synchronization problem. 
\begin{dfn}
Let $\mathcal{L}(\meas)$ be a \emph{loss} functional of the form:
\begin{align}\label{eq:loss}
	\mathcal{L}(\meas)\triangleq \sum\nolimits_{(i,j)\in \Edges} \mathcal{D}(g_{ij}(\meas), \boldsymbol{\mu}_{ij})
\end{align}
where $\mathcal{D}$ denotes a divergence that measures the discrepancy between  two probability distributions and let $\mathcal{R}$ be a \emph{regularization} functional imposing structure on $\meas$.
The measure synchronization problem is formally defined as:
\begin{align}
\nonumber &\min_{\meas} \>\>\> \mathcal{L}(\meas) +\mathcal{R}(\meas) \\
&\st \quad [w_i^{1}, \dots, w_i^{K_i}]^\top \in \mathcal{C}_i, \quad \forall i \in \{1,\dots,n\} \label{eqn:meassync}
\end{align}
where $\mathcal{C}_i \subseteq \mathbb{R}_+^{K_i}$ denotes a constraint set where the weights of the absolute measure $\meas_i$ are restricted to live in. 
\end{dfn}
Note that in both (HE) and (LE), $\meas$ can be completely recovered from the marginals $\meas_i$ and that an additional constraint on the weights is needed in (LE): $w_i^{k}= w_{j}^{k}$ for all $1\leq i,j\leq n$. 

Our problem formulation resembles the recent optimal transport-based implicit generative modeling problems~\cite{arjovsky2017wasserstein,kolouri2018sliced,tolstikhin2017wasserstein}. The main differences in our setting are as follows: (i) the choice of the composition functions $g_{ij}$, which is specific to the pose synchronization problem, (ii) the problem is purely nonparametric in the sense that the absolute measures $\{\meas_i\}_i$ are not defined through any parametric mapping. 
We will be interested in two choices for the loss  $\mathcal{D}$:

\vspace{-7pt}

\paragraph{Case 1 - \textit{Sinkhorn-Sync}.} Due to its recent success in various applications in machine learning~\cite{arjovsky2017wasserstein} and its favorable theoretical properties~\cite{genevay2018sample}, in our first configuration, we choose the loss functional $\mathcal{D}$ to be the squared Sinkhorn divergence $\mathcal{S}_{c,\alpha}^2$ with finite $\alpha$, as defined in~\cref{eq:sinkhorn}. In this setting, we impose the absolute measures $\meas_i$ to be probability measures, i.e.\ we set each $\mathcal{C}_i$ to be the probability simplex $\Sigma_{K_i}$, hence imposing the sum $\sum_{k=1}^{K_1} w_i^{(k)}$ to be equal to $1$. Thanks to these constraints, in this setup we do not need to use regularization, i.e.\ we set $\mathcal{R}(\mu) = 0$.  

\insertimageStar{1}{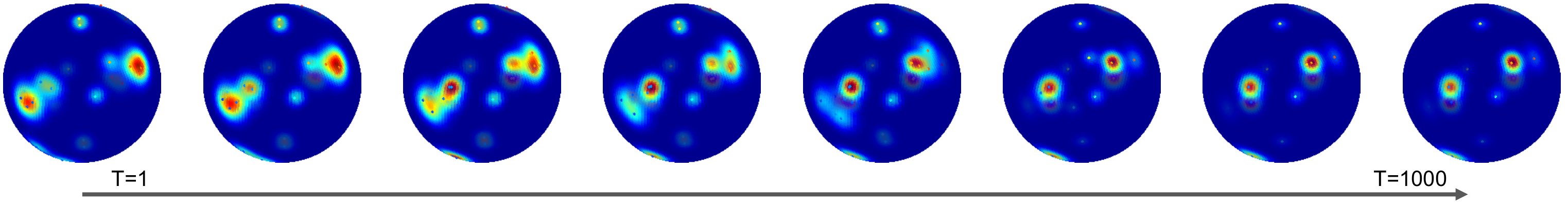}{Density estimates for a single camera visualized over time. We sum up the probability density functions of Bingham distributions centered on each quaternion and marginalize over the angles to reduce the dimension to 3. This way we are able to plot the density estimate for all particles of a given camera and track it across iterations. Note that in here we use MMD distance with an unconstrained RPGD.\vspace{-3mm}}{fig:evolution}{t!}

\vspace{-7pt}
\noindent \paragraph{Case 2 - \textit{MMD-Sync}.} In our second configuration, we consider a setting which will enable us to analyze the proposed Riemannian particle descent algorithm, which will be defined in the next section. In particular, we consider the loss functional is given by the  MMD distance. In this setup, we constrain the absolute measures $\meas_i$ to be positive measures by choosing the constraint sets to be $\mathcal{C}_i = \mathbb{R}_+^{K_i}$. Accordingly, we use regularization on the weights of $\meas_i$, that is given as follows:
\begin{align}
\mathcal{R}(\meas) = \lambda \sum\nolimits_{i=1}^n\sum\nolimits_{k_i=1}^{K_i} w_i^{(k_i)}
\label{eqn:reg_def}
\end{align}
where $\lambda> 0$ is a hyper-parameter. 

By relaxing the constraint on the weights, this problem can be written as an unconstrained optimization problem, which favors `sparse' solutions, in the sense that we enforce the weights to have small values, which in turn forces the solution to have only a few point masses with large weights. 
We will show that in the (LE) setting and by using the structure of MMD, this setup is closely related to the recent sparse optimization problem in the space of measures~\cite{chizat2019sparse} and hence inherits its theoretical properties.

\vspace{-7pt}

\paragraph{Choice of the ground cost.} As for the ground cost function $c$ required for the Sinkhorn divergence and MMD, we investigate two options. The first option is simply choosing the squared Euclidean distance between the point masses, $c(\hat{\q}_{ij}^{(k)} ,\q_{ij}^{(l)}) = \|\hat{\q}_{ij}^{(k)} - \q_{ij}^{(l)} \|^2_2$, where $\hat{\q}_{ij}^{(k)}$ and $\q_{ij}^{(k)}$ are given in \eqref{eqn:comp_fn} and \eqref{eqn:rel_meas}, respectively. This yields a positive definite MMD kernel. On the other hand, for the quaternions which are naturally endowed with a Riemannian metric, the second option is the geodesic distances as defined in \cref{def:riemann_dist}, i.e.\  $c(\hat{\q}_{ij}^{(k)} ,\q_{ij}^{(l)}) = d(\hat{\q}_{ij}^{(k)}, \q_{ij}^{(l)})^p$, $1\leq p\leq 2$.
This second choice leads to a positive definite MMD kernel $k$ (\cf~\cref{dfn:mmd}) only when $p=1$. Positivity of the kernel is a theoretical requirement for both the Sinkhorn divergence and the MMD to be a distance metric, 
however, it turns out that the computation of the gradients when $p=1$ becomes numerically unstable. Hence, as a remedy, we follow the robust optimization literature \cite{aftab2014generalized} and consider $p > 1$.
This approach finds a good balance in terms of numerical stability: for $p$ slightly larger than $1$, we lose from the positive definiteness of the kernel $k$; however, the numerical computation of the gradients becomes numerically stable. We have observed in practice that small values of $p$ \ie $1+\epsilon< p \leq 1.4$ would not harm our algorithm.

\paragraph{Remarks.}
We note that in the (HE) case, the cost functional requires a quadratic pairing of the elements in the absolute measures to estimate the relative distribution. That is, any quaternions $\q_i \sim \meas_i$ and $\q_j \sim \meas_j$ are assumed to be potentially correct candidates for the orientation of cameras $i$ and $j$, respectively and any relative rotation computed from those would be a valid candidate for explaining $\meas_{ij}$.

\section{Riemannian Particle Gradient Descent}
Due to the non-convex nature of our problem, solving \eqref{eqn:meassync} is a highly challenging task. In this section, we will present our method built upon the recently proposed particle gradient descent (PGD) algorithms which are designed to solve optimization problems in the space of measures \cite{chizat2019sparse,liutkus2019sliced,arbel2019maximum,liu2017stein}. Such PGD algorithms are often derived by first considering a continuous-time gradient flow in the Wasserstein spaces (i.e., the metric space associated with the Wassertein distance), given as follows:
\begin{align}
\mathrm{d} \meas(t) = -\nabla_{\mathcal{W}} \Bigr( \mathcal{L}(\meas(t)) + \mathcal{R}(\meas(t)) \Bigr)\mathrm{d}t, \label{eqn:grad_flow}
\end{align}
where $t$ denotes the (artificially introduced) `time' and $\nabla_{\mathcal{W}}$ denotes a notion of gradient in the Wasserstein space \cite{santambrogio2017euclidean}. Under appropriate conditions on $\mathcal{L}$ and $\mathcal{R}$, one can show that the trajectories of the flow $(\meas(t))_{t \geq 0}$ converge to stable points of the optimization problem given in \eqref{eqn:meassync}~\cite{santambrogio2017euclidean,ambrosio2008gradient}. 

Since the trajectories of~\eqref{eqn:grad_flow} are continuous in time, in general, they cannot be directly simulated on a computer. Hence, similar to the conventional Euclidean gradient descent algorithm that can be viewed as a discretization of a gradient flow in the Euclidean space, we can discretize the gradient flow in \eqref{eqn:grad_flow} to be able to simulate the trajectories in discrete-time. Such discretizations give rise to the PGD algorithms, which are in general applicable to measures that are defined in Euclidean spaces. Unfortunately, standard PGD algorithms often cannot be directly applied to our problem since our measures are defined on a Riemannian manifold. Therefore, similar to \cite{chizat2019sparse}, we consider a modified PGD algorithm, which we call the Riemannian PGD (RPGD). RPGD consists of two coupled equations. The first equation updates the location of the particles $\q_i^{(k)}$ on the quaternion manifold and doesn't depend on $\mathcal{R}(\meas)$:
\begin{align}\label{eq:quaternion_update}
\q_i^{(k)} \gets \Exp_{\q_i^{(k)}} \Bigr(- \eta_q(w_i^{(k)})\nabla_{\q_i^{(k)}} \mathcal{L}(\meas) \Bigr).
\end{align}
Here, $\Exp$ denotes the exponential map defined in Section~\ref{sec:OT}  and $\eta_q(w_i^{(k)})$ is an adaptive step-size that depends on the weight of the particle $\q_{i}^{(k)}$ only. We consider either a constant step-size $\eta_q(w_i^{(k)}) = \eta_q$ for the Sinkhorn divergence or a step-size that is inversely proportional to the weight, i.e. $\eta_q(w_i^{(k)}) = \frac{\eta^0_q}{ w_i^{(k)}} $ for the MMD as proposed in \cite{chizat2019sparse}. We will discuss this choice in more details later in this section. 
We also provide the analytical form of the gradients with respect to $\q_i^{(k)}$ in the supplementary document. 

The second equation updates assumes a parametrization of the weights $w_i^{(k)}$ of the form $w_i^{(k)} = (\beta_i^{(k)})^2$ and uses a smaller step-size $\eta_{\beta}$. We recall the two cases considered: the \textit{constrained case} (Sinkhorn-Sync) and the \textit{unconstrained case} (MMD-Sync). In the latter, the weights are not constrained to sum to $1$ and the penalty term  $\mathcal{R}(\meas(t))$ ensures they don't diverge:
\begin{align}
	\beta_i^{(k)} \gets \beta_i^{(k)}- \eta_{\beta} \nabla_{\beta_i^{(k)}} \bigl(\mathcal{L}(\meas)+\mathcal{R}(\meas)\Bigr)
\end{align}
On the ohter hand, Sinkhorn-Sync enforces the weights to sum to $1$ by performing optimization on the unit hyper-sphere:
\begin{align}
\beta_i^{(k)} \gets \Exp_{\beta_i^{(k)}} \Big(-  \eta_{\beta} \nabla_{\beta_i^{(k)}} \bigl(\mathcal{L}( \meas)  \bigr) \Big).
\end{align}
We have implemented the proposed algorithm in Python using PyTorch~\cite{paszke2017automatic} for GPU support. We provide a pseudocode in the supplementary document and release the full source code at {\footnotesize\url{synchinvision.github.io/probsync/}}. 

\begin{figure}[t]
\centering
\includegraphics[width=\linewidth]{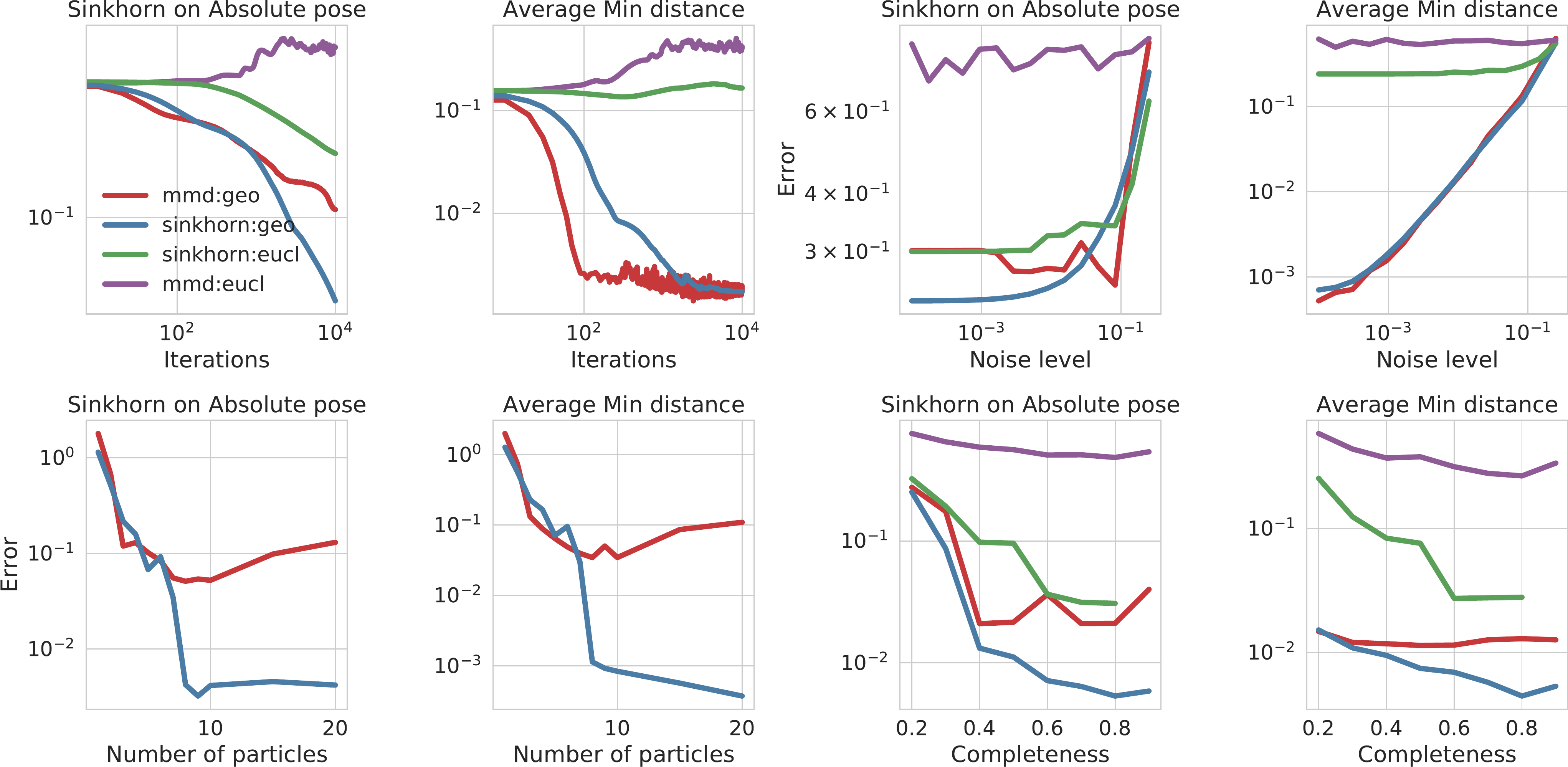}
\caption{Controlled experiments in synthetic data: We show the performance (Sinkhorn and minimum distances) of our method against varying factors of noise, graph sparsity and number of particles. In all experiments the ground truth distribution has 3 modes.}\label{fig:sensitivity}\vspace{-3mm}
\end{figure}

\paragraph{Convergence result for regularized MMD.} 
In this section, we will analyze the convergence behavior of the proposed RPGD algorithm on a special case of MMD-Sync. In particular, we restrict ourselves to the \textit{Low entropy} case and where we choose MMD as the loss function and we choose the regularizer as in \eqref{eqn:reg_def}.
This formulation enables us to analyze RPGD as a special case of the sparse optimization framework considered in \cite{chizat2019sparse}. Thanks to the convexity of the MMD loss, we will be able to show that RPGD will converge to the global optimum under the following conditions:
\begin{assumption}
The kernel k is twice Fr\'{e}chet differentiable, with locally Lipschitz second-order derivatives.
\vspace{-5pt}
\end{assumption}
\begin{assumption}
\label{asmp:op}
 Problem \cref{eqn:meassync} admits a unique global minimizer which is of the form $\meas^\star= \sum_{k=1}^{K^\star} v^\star_k \delta_{\q_k^\star}$ with $K^\star < K=K_1 = \cdots = K_n$. 
 \vspace{-5pt}
\end{assumption}
\begin{assumption}
\label{asmp:nondeg}
The minimizer $\meas^\star$ is non-degenerate in the sense of \cite{chizat2019sparse} Assumption 5.
\vspace{-5pt}
\end{assumption}
\begin{assumption}
	The weights $w_i^{(k)}$ are parametrized as: $w_i^{(k)} = w_j^{(k)} = (\beta^{(k)})^2$ for all $1\leq i,j\leq n$ and the adaptive step-size in \cref{eq:quaternion_update} is chosen to be ${\eta^0_q}/{ w_i^{(k)}}   $ for some $\eta^0_{q}>0$.
	\vspace{-5pt}
\end{assumption}
\begin{thm}
\label{thm:global}
Consider the LE setting and Case 2 defined above. Assume that \textbf{H}1-4 hold. Then, the RPGD algorithm converges to the global optimum of \eqref{eqn:meassync} with a linear rate.
\vspace{-2pt}
\end{thm}
\begin{proof}
The result is a consequence of \cite{chizat2019sparse}, Thm. 3.9.
\end{proof}
We provide a more detailed definition of~\cref{asmp:nondeg} in the supplementary. This result shows that, even though our problem is non-convex, the RPGD algorithm will be able to find the global optimum thanks to over-parametrization (cf.\cref{asmp:op}). We also note that the uniqueness condition in \cref{asmp:op} does not directly hold due to the antipodal symmetry of the quaternions. We circumvent this issue by making sure that the scalar part of each quaternion is positive.

\section{Experimental Evaluation}
\label{sec:exp}
We evaluate our algorithm qualitatively and quantitatively on a series of synthetic and real test beds. 
We use the HE case in all the settings except the PnP ambiguity application. Unless stated otherwise, we fix $\alpha=0.05$ and $p= 1.2$.
We note that in all the settings, in order to eliminate the gauge ambiguity, we set the absolute measure of the first camera to its true distribution. This, for the case of a single particle, is $\q_1^{(1)}=[1,0,0,0]^\top$ without loss of generality, while for the K-best prediction scenarios, we are obliged to use the ground truth particles. We leave it as a future study to remove this requirement. 

\subsection{Evaluations on synthetic data.}
Similar to~\cite{birdal2018bayesian}, we use a controlled setup to carefully observe the behaviour of our algorithm and its variants. To do so, we begin by synthesizing a dataset of $N=10$ cameras where we could have access to ground truth absolute and relative pose distributions. We use $K_i=3$ ground truth modes and from those generate $K_{ij}=9$ relative rotation particles. Optionally, we corrupt the ground truth with gradual noise. To this end, a perturbation quaternion is applied on the relative quaternion. We set the learning rate to $\eta=0.01$. We report the Sinkhorn-error attained as well as the actual geodesic distance (\textit{Average Min distance}) between the matching particles. It is reassuring that these metrics behave similarly (\cf~\cref{fig:evolution} and \cref{fig:sensitivity}). We run our algorithm for $t_{max}=10000$ iterations until no improvement in the error is recorded. We analyze four variants of our algorithm: (1) MMD loss with Euclidean kernel (mmd:euc), (2) MMD loss with Geodesic kernel (mmd:geo), (3) Sinkhorn divergence with Euclidean kernel (sinkhorn:euc) and (4) Sinkhorn divergence with Geodesic kernel (sinkhorn:geo). 

\setlength{\intextsep}{3.5pt}%
\begin{table*}[t]
  \centering
  \caption{Results for finding the single-best solution on EPFL Benchmark~\cite{strecha2008benchmarking}. We run a specific case of our algorithm where both $k_{abs}=1$ and $k_{rel}=1$, while all the other methods are specialized for this special case. We also use random initialization whereas~\cite{govindu2004lie,torsello2011multiview,birdal2018bayesian} are initialized from a minimum spanning tree solution. The table reports the well accepted mean angular errors in degrees.}
  \resizebox{\textwidth}{!}{%
    \begin{tabular}{lccccccc}
          & Ozyesil~\etal~\cite{ozyesil2015} & R-GODEC~\cite{arrigoni2014robust} & Govindu~\cite{govindu2004lie} & Torsello~\cite{torsello2011multiview} & EIG-SE(3)~\cite{arrigoni2015spectral} & TG-MCMC~\cite{birdal2018bayesian} & {Ours} \\
    \midrule
    HerzJesus-P8 & 0.060 & 0.040 & 0.106 & 0.106 & 0.040 & 0.106 & 0.057 \\
    HerzJesus-P25 & 0.140 & 0.130 & 0.081 & 0.081 & 0.070 & 0.081 & 0.092 \\
    Fountain-P11 & 0.030 & 0.030 & 0.071 & 0.071 & 0.030 & 0.071 & 0.024 \\
    Entry-P10 & 0.560 & 0.440 & 0.101 & 0.101 & 0.040 & 0.090 & 0.190 \\
    Castle-P19 & 3.690 & 1.570 & 0.393 & 0.393 & 1.480 & 0.393 & 0.630 \\
    Castle-P30 & 1.970 & 0.780 & 0.631 & 0.629 & 0.530 & 0.622 & 0.753 \\
    \midrule
    Average & 1.075 & 0.498 & 0.230 & 0.230 & 0.365 & 0.227 & 0.291 \\
    \end{tabular}%
    }
  \label{tab:epfl}\vspace{-3mm}
\end{table*}%

\vspace{1mm}\noindent\textbf{Iterations.}
We see that despite the theoretical guarantees, in practice minimizing the Sinkhorn loss yields better minimum. This is due to the fact that sometimes it is not possible to satisfy all the assumptions presented above~\cref{thm:global}. In fact, for all the other experiments, we have observed a better behaviour of Sinkhorn divergences and will argue for this method of choice. It is also noteworthy that MMD with a Euclidean kernel is not guaranteed to converge with the learning rates tested.

\vspace{1mm}\noindent\textbf{Noise resilience.}
We see that it is important to use the true metric (geo) when it comes to dealing with noise. Both MMD and Sinkhorn losses can handle reasonable noise levels. Note that while small levels of noise ($\sim 0.01$) are tolerable, the measure synchronization problem gets more difficult with the increasing noise: we observe an exponential-like increase in the error attained after $\sigma>0.01$.

\vspace{1mm}\noindent\textbf{How do we scale with the number of particles?} 
The behavior of our algorithm when we don't know the exact number of modes in the data is of curiosity. We see that to we can recover the $K_{ij}=9$ better when we use $K_i>9$ absolute particles during estimation. This is not surprising and strongly motivated by the theory~\cite{chizat2019sparse} which can give provable guarantees only when this number is over-estimated. Nevertheless, our experiments show that as long as we overestimate, the algorithm is insensitive to the particular choice of this number, and what is affected is the runtime. 

\vspace{1mm}\noindent\textbf{Robustness against graph sparsity.} 
We observe that graphs of $50\%$ completeness (connectedness) are sufficient for our algorithm to get good results. We have diminishing returns from making the graph more connected. Nevertheless, Sinkhorn:geo seems to benefit the most from the increasing number of edges. 

\subsection{Evaluations on real data}
\noindent\textbf{Single best solution.}
Classical multiple rotation averaging (or synchronization) is a special case of the problem we pose. When our algorithm is run with a single particle both for the absolute and relative rotations, we recover the solution to the classical problem. While our aim is not to explicitly tackle this challenge, being able to handle such special case is a sanity check for us. We present our results in~\cref{tab:epfl} where we report the mean angular errors in degrees on the standard EPFL Benchmark dataset~\cite{strecha2008benchmarking}. We observe that our performance is on par with state of the art on this real dataset, even when starting from a random initialization. Moreover, unlike TG-MCMC~\cite{birdal2018bayesian} our approach is non-parametric and we do not perform any scene specific tuning. Note that we do not process the translation information which is an additional cue for the other methods.

\insertimageC{1}{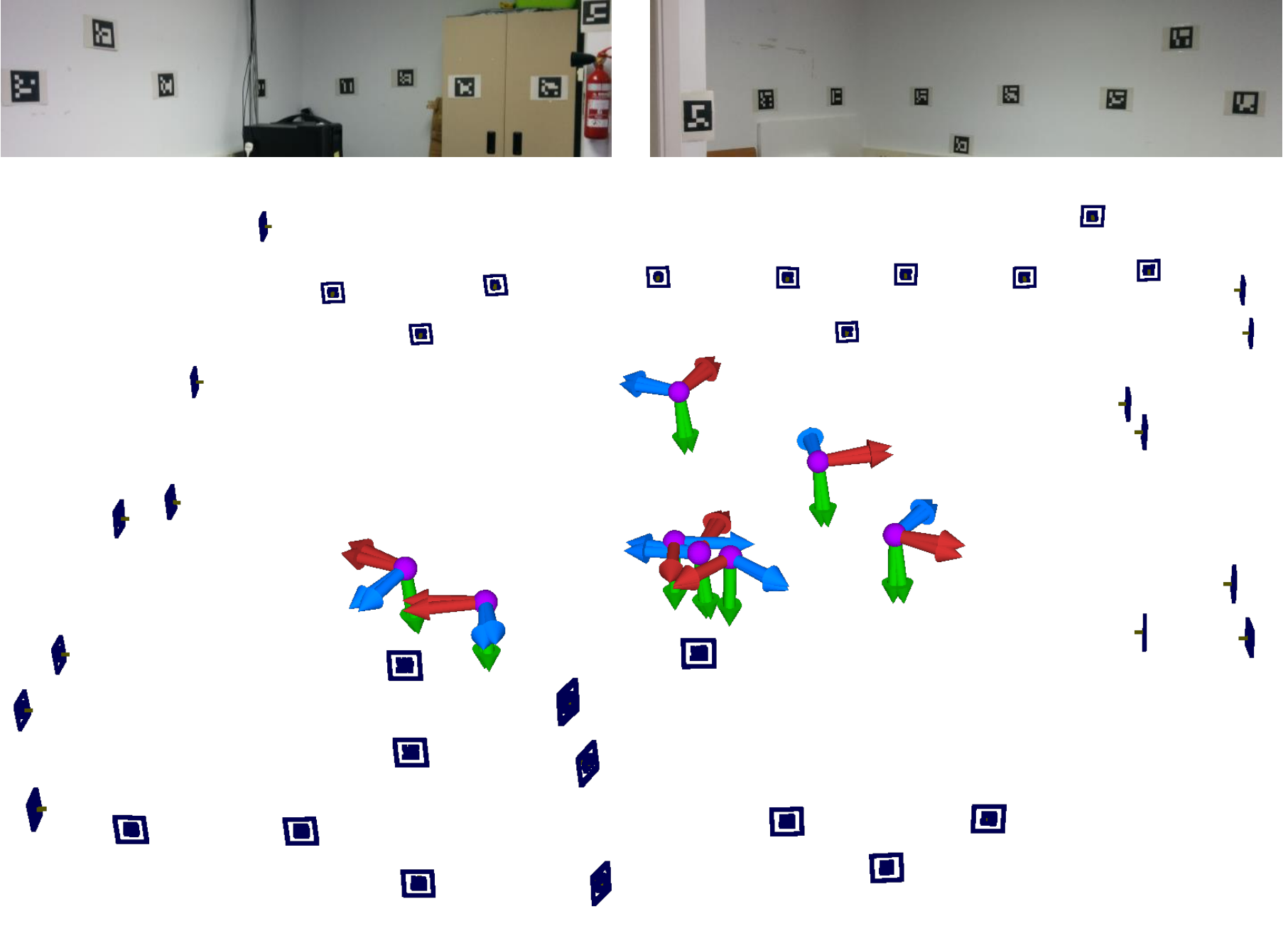}{Our estimation vs. ground truth particles in the markers dataset. Estimated orientations are displayed at common camera centers of RGB frames. On the top row we show two images from the dataset of~\cite{munoz2018mapping}. 3D location of the markers are shown as black squares. The small deviation is the error we make.\vspace{-4mm}}{fig:markers}{t!}

\vspace{1mm}\noindent\textbf{Resolving the PnP ambiguity. \,} Our algorithm can be used to treat many multiview vision problems with their inherent ambiguities. One such ambiguity arises in the Perspective N-point problem~\cite{lepetit2009epnp} where the single view camera pose estimation from a planar target admits two solutions. While targeted synchronization approaches do exist~\cite{ch2019resolving}, our framework can be a natural fit. We test our algorithm on the multiview dataset of MarkerMapper~\cite{munoz2018mapping}. We first estimate the pairwise poses between cameras from randomly selected markers and keep both of the possible solutions. We then use 2 particles per camera with joint coupling to recover the entirety of the plausible solution space. Further details are given in the suppl. document. We attain a minimum error of $7^\circ$, a quantity that is on par with PnP-specific synchronizers~\cite{ch2019resolving}. We present qualitative results from our estimations in~\cref{fig:markers}.

\vspace{-2mm}\section{Conclusion}\vspace{-2mm}
\label{sec:conclude}
We introduced the problem of \textit{measure syncronization}: synchronizing graphs whose edges are probability measures. We then extended it to Riemannian measures and formulated the problem as maximization of a novel cycle-consistency notion over the space of measures. 
After formally defining the problem, we developed a Riemannian particle gradient descent algorithm for estimating the marginal distributions on absolute poses. In addition to being a practical algorithm, by drawing connections to recent optimization methods, we showed that the proposed algorithm converges to the global optimum of a special case of the problem. We showed the validity of our approach via qualitative and quantitative experiments.
Our future work will address better design of composition functions $g$, exploration of sliced OT distances \cite{nadjahi2019asymptotic,nadjahi2020statistical,kolouri2020generalized}, and applications to tackling more ambiguities in vision such as essential matrix estimation.

\noindent\footnotesize\textbf{Acknowledgements}. We thank Guilia Luise, L\'ena\"{\i}c Chizat and Justin Solomon for fruitful discussions, Mai Bui for the aid in data preparation and Fabian Manhardt in visualizations. This work is partially supported by Stanford-Ford Alliance, NSF grant IIS-1763268, Vannevar Bush Faculty Fellowship, Samsung GRO program, the Stanford SAIL Toyota Research Center, and the French ANR grant ANR-16-CE23-0014.

{\small

}

\onecolumn
\setcounter{section}{0}
\renewcommand\thesection{\Alph{section}}
\newcommand{\suppsection}{\subsection}
\clearpage
\begin{center}
\textbf{\large Synchronizing Probability Measures on Rotations via Optimal Transport \\ Supplementary Material}
\end{center}
\makeatletter

This document supplements our main paper entitled \textit{Synchronizing Probability Measures on Rotations via Optimal Transport}. In specifics, we provide a background section on the technical part, the explicit form of the gradients required by the algorithm, a more detailed explanation on our assumptions and the composition functions. We also include the pseudocode of our method and two additional experiments: (1) on a real SfM dataset, (2) on the mug sequence shown in the main paper.

\section{Connection to Maximum Likelihood Estimation (MLE) and Markov Random Fields (MRF)}
Rotation synchronization has been studied in the literature under the name \emph{multiple rotation averaging}. The standard single-particle based methods such as SE-Sync~\cite{rosen2016se} assume a unimodal Gaussian/Langevin distribution. There are two caveats with that. First, the  classical  approaches cannot yield explicit uncertainty estimates, and second a unimodal distribution cannot capture ambiguities that can be multimodal. 
DISCO has tackled this problem via MRFs and loopy belief propagation~\cite{crandall2012sfm}. 
In fact our formulation is similar when the nodes are assumed to have uniform prior. Yet, like K-best syncronization~\cite{sun2019}, DISCO requires a single pairwise potential, as opposed to the multimodal distributions we have. To the best of our knowledge, such MRF methods have not been extended to work in our setting. Note that differently to all those our approach falls in the non-parametric inference.

\section{Optimal Transport on Riemannian Manifolds}
Here we denote by $\X$ a Riemannian manifold, which can be for instance the set of unit quaternions $\mathbb{H}$.
\subsection{Optimal Transport}
For two given probability distributions $\nu$ and $\mu$ in $\mathcal{P}_2(\X)$, we denote by $\Pi(\nu,\mu)$ the set of couplings between $\nu$ and $\mu$, i.e.:  $\Pi(\nu,\mu)$ contains all joint distributions $\pi$ on $\X\times \X$ such that if $(X,Y) \sim \pi $ then $X \sim \nu $ and $Y\sim \mu$. The $2$-Wasserstein distance on $\mathcal{P}_2(\X)$ is defined by means of an optimal coupling between $\nu$ and $\mu$:
\begin{align}\label{eq:wasserstein_2}
W_2^2(\nu,\mu) := \inf_{\pi\in\Pi(\nu,\mu)} \int \left\Vert x - y\right\Vert^2 \diff \pi(x,y) \qquad \forall \nu, \mu\in \mathcal{P}_2(\X)
\end{align}
It is a well established fact that such optimal coupling $\pi^*$ exists \cite{Villani:2009,Santambrogio:2015} . Moreover, $\text{W}_2$ enjoys a dynamical formulation which gives it an interpretation as the length  of the shortest path connecting $\nu$ and $\mu$ in probability space.  It is summarized by the celebrated Benamou-Brenier formula (\cite{benamou2000computational}):
\begin{align}\label{eq:benamou-brenier-formula}
W_2(\nu,\mu) = \inf_{(\rho_t,V_t)_{t\in[0,1]}}\int_0^1 \int \Vert V_t(x) \Vert^2 d\rho_t(x),
\end{align}
where the infimum is taken  over all couples  $\rho$ and $v$ satisfying  a continuity equation with boundary conditions:
\begin{align}\label{eq:continuity_equation}
\partial_t \rho_t + \mathrm{div}(\rho_t V_t ) = 0, \qquad 
\rho_0 = \nu, \quad \rho_1 = \mu.
\end{align}
The above equation expresses two facts, the first one is that $-\mathrm{div}(\rho_t V_t)$ reflects the infinitesimal changes in $\rho_t$ as dictated by the vector field (also referred to as velocity field) $V_t$, the second one is that the total mass of $\rho_t$ does not vary in time as a consequence of the divergence theorem. Equation \cref{eq:continuity_equation} is well defined in the distribution sense even when $\rho_t$ does not have a density and $V_t$ can be interpreted as a tangent vector to the curve $(\rho_t)_{t\in[0,1]}$.

In \cref{sec:particle_descent} we will see that the continuity equation in \cref{eq:continuity_equation1} without terminal condition $\rho_1=\mu$ and for a well chosen vector field  $V_t$ leads to a gradient flow in probability space.
\subsection{First variation of a functional}
Here we introduce the notion of first variation of a functional $\mathcal{F}$ which will be crucial to define the Wasserstein gradient flow in \cref{sec:gradient_flow}. We then provide explicit expressions of this first variation in the case of the MMD and  sinkhorn divergence.

Consider a real valued functional $\F$ defined over $\mathcal{P}_2(\X)$. We call $\frac{\partial{\F}}{\partial{\nu}}$ if it exists, the unique (up to additive constants) function such that $\frac{d}{d\epsilon}\F(\nu+\epsilon  (\nu'-\nu))\vert_{\epsilon=0}=\int\frac{\partial{\F}}{\partial{\nu}}(\nu) (\diff \nu'-\diff \nu) $ for any $\nu' \in \mathcal{P}_2(\X)$. For a fixed $\nu$, the function $\frac{\partial{\F}}{\partial{\nu}}(\nu)$ is a real valued function defined on $\X$ and is called the first variation of $\F$ evaluated at $\nu$. 

In the case of the squared MMD, a simple expression is obtained by direct calculation:
\begin{align}
	\frac{\partial \text{MMD}^2(\nu,\mu)}{\partial \nu}(\nu)(x) = 2\left(\int k(x,y)d\nu(y) -  \int k(x,y)d\mu(y)\right).
\end{align}
This can be easily estimated using samples from both $\mu$ and $\nu$.

The first variation of the Sinkhorn is more involved. We first recall the expression of the Sinkhorn distance $d_{c,\alpha}(\mu,\nu)$ in terms of the optimal potential functions:
\begin{align}
    d_{c,\alpha}(\mu,\nu) = \int f(x)d\mu(x) + \int g(x)d\nu(x)
\end{align}
where $f$ and $g$ are unique up to an additive constant \cite[Proposition 1]{feydy2019}. In practice, given samples $(X_i)_{1\leq i \leq N}$ and $(Y_i)_{1\leq i\leq N}$ from  $\nu$ and $\mu$, $f$ and $g$ can be estimated on those values using the iterative sinkhorn algorithm, this provides vectors  $f_i$ and $g_i$ such that $f_i\sim f(X_i) $ and $g_i \sim g(Y_i) $.

The first variation of the Sinkhorn distance  is simply given by  differentiating wrt $\nu$:
\begin{align}
   \frac{ \partial d_{c,\epsilon}(\mu,\nu)}{\partial \nu}(\nu)(x) = g(x)
\end{align}
However, $g$ needs to be evaluated at arbitrary points $x$, while the Sinkhorn algorithm only provides the values $g_i$ and $f_i$ at the sample points $X_i$ and $Y_i$. This is not an issue as noted in \cite{genevay2018sample,feydy2019}. Indeed, $f$ and $g$ are related by the equation:
\begin{align}
    g(x) = -\epsilon \log(\int  \exp( \frac{f(y) - c(x,y)}{\epsilon}  )d\nu(y) )
\end{align}
Hence, $g$ can be estimated by replacing the expectation by the empirical one and using the estimated values $f_i$ at the sample points $Y_i$  :
\begin{align}
    \hat{g}(x) = -\epsilon \log(\frac{1}{N}  \exp( \frac{f_i - c(x,Y_i)}{\epsilon}  ) ).
\end{align}
Finally, the variation of the Sinkhorn divergence, is obtained by summing those of each of it's components:
\begin{align}
     \frac{ \partial \mathcal{S}_{c,\epsilon}(\mu,\nu)}{\partial \nu}(\nu)(x) = 2*\frac{ \partial d_{c,\epsilon}(\mu,\nu)}{\partial \nu}(\nu)(x) - \frac{ \partial d_{c,\epsilon}(\nu,\nu)}{\partial \nu}(\nu)(x).
\end{align}
\subsection{Wasserstein gradient flow}\label{sec:gradient_flow}

The formal gradient flow equation associated to a functional $\F$ can be written (see \cite{carrillo2006contractions}, Lemma 8 to 10):
\begin{equation}\label{eq:continuity_equation1}
\frac{\partial \nu_t}{\partial t}= \mathrm{div}( \nu_t \nabla \frac{\partial \F}{\partial \nu_t})
\end{equation}
where $\mathrm{div}$ is the divergence operator and $\nabla \frac{\partial \F}{\partial \nu}(x)$ is the Riemannian gradient of $\frac{\partial}{\partial \nu} \F(x)$ which is an element of the tangent space of $\X$ at point $x$. 
It can be shown that the probability distributions $\nu_t$ decrease $\F$ in time. More precisely the following energy dissipation equation holds under mild regularity conditions on $\F$:
\begin{align}
	\F(\nu_t) = - \int \Vert  \nabla \frac{\partial \F}{\partial \nu}(\nu_t)(x)\Vert^2 d\nu_t(x).   
\end{align}
$\F(\nu_t)$ is a decreasing function in time,  hence the interpretation of $\nu_t$ as a gradient flow of $\F$.

\subsection{Riemannian Particle Descent}\label{sec:particle_descent}
 The equation in \ref{eq:continuity_equation1} admits an equivalent expression in terms of particles which will be useful in practice:
\begin{align}
	\frac{X_t}{dt} =   -\nabla \frac{\partial \F}{\partial \nu}(\nu_t)(X_t)
\end{align}
A discretization in time and space can be performed in the following way: Given $N$ initial particles $(X_0^{i})_{1\leq i\leq N}$, and a step-size $\gamma$, the following update rule can be used:
\begin{align}
	X_{t+1}^{i} = \exp_{X_t^{i}}(- \gamma \nabla \frac{\partial \F}{\partial \nu}(\hat{\nu}_t)(X_t^{i})  )
\end{align}
where $\hat{\nu}_t$ is the particle measure at time $t$:  $  \hat{\nu}_t = \frac{1}{N} \sum_{i=1}^N  \delta_{X_t^{i}}$ and $\exp$ is the exponential map associated to the manifold $\X$. We refer to \cref{sec:exponential_map} for a closed form expression of the exponential map in the case of unit-quaternions.
We provide a pseudocode for the proposed RPGD algorithm in Algorithm~\ref{algo:flow}. We also release our implementation under: \url{https://synchinvision.github.io/probsync}.

     \begin{algorithm2e}[b]
         \SetInd{0.1ex}{1.5ex}
         \DontPrintSemicolon
         \SetKwInOut{Input}{input}
         \SetKwInOut{Output}{output}
         \Input{Relative measures $\{\meas_{ij}\}_{i,j=1}^n$}
         \Output{Absolute measures $\{\meas_i\}_{i=1}^n$}
         {\color{purple} \small \tcp{Initialize the particles}}
         $\q_i^{(k)} \sim \meas_i$, \hspace{50pt} $i = 1,\dots,n$, $\quad k = 1,\dots,K_n$\\
     {\color{purple} \small \tcp{Iterations}}
     \For{$t = 0,\dots T-1$}
     {
     	{\color{purple} \small \tcp{For all cameras}}
        \For{$i = 1,\dots,n$}
        {
        {\color{purple} \small \tcp{Update the positions of the particles}}
        $\q_i^{(k)} \gets \Exp_{\q_i^{(k)}} \Bigr(- \eta_q(w_i^{(k)})\nabla_{\q_i^{(k)}} \mathcal{L}(\meas) \Bigr)$ \hspace{50pt} $k = 1,\dots,K_n$\\
        {\color{purple} \small \tcp{Update the weights of the particles -- Unconstrained case}}
        $\beta_i^{(k)} \gets \beta_i^{(k)}- \eta_{\beta} \nabla_{\beta_i^{(k)}} \bigl(\mathcal{L}(\meas)+\mathcal{R}(\meas)\Bigr)$ \hspace{20pt} $k = 1,\dots,K_n$  \\
        {\color{purple} \small \tcp{Update the weights of the particles -- Constrained case}}
        $\beta_i^{(k)} \gets \Exp_{\beta_i^{(k)}} \Big(-  \eta_{\beta} \nabla_{\beta_i^{(k)}} \bigl(\mathcal{L}( \meas)  \bigr) \Big)$
        }
     }
         \caption{Riemannian Particle Gradient Descent for Measure Synchronization}
         \label{algo:flow}
     \end{algorithm2e}

\section{Analytic Form of the Gradients}
In this section, we provide the analytical forms of the gradients required by our algorithm. We first recall the expression of the normalized logarithm of a unit quaternion $\x:=(a,\vb)$ which is given by:
\begin{align}
	\frac{\log(\x)}{\Vert \x \Vert } = (0,\frac{\vb}{\Vert \vb \Vert})
\end{align}
we also right $\log_{\x}(\y) := \log(\x^{-1}\y)$.
	A subgradient of the Riemannian distance $d(\cdot)$ is given as follows:
\begin{align}
\label{eq:dgrad}
\nabla_{\x} d(\x \in \QH, \y \in \QH) &= 
\begin{cases}
-\sign(\langle \x,\y\rangle )\frac{\log_{\x}(\y)}{\|\log_{\x}(\y)\|}\equiv \Big(0, -s\frac{\vb}{\|\vb\|} \Big) & \x\neq\y \\
\hfil 0 & \x=\y \\
\end{cases}
\end{align}
where $\vb$ denotes the imaginary part of $\x^{-1}\y$ and $s := \sign(\langle \x,\y \rangle) $ is the sign of dot product between $\x$ and $\y$ with the convention that $s=1$ if the dot product is $0$.

By using this formulation and the chain rule of differentiation, we obtain the  gradient required by RPGD which is given by Proposition \ref{prop:grad_dist}. 
We finally combine this gradient with the gradient of the Sinkhorn divergence or the gradient of MMD by using autodiff.

\begin{prop}
\label{prop:grad_dist}
The gradient of $d(\q_i\q_j^{-1}, \q_{ij}  )$ w.r.t. $\q_i$ and $\q_j$ is given by:
	\begin{align}
		\nabla_{\q_i} d(\hat{\q}_{ij},\q_{ij}))&=   \q_j^{-1} \nabla_{\hat{\q}_{ij}} d(\hat{\q}_{ij},\q_{ij})\q_j\\
		 \nabla_{\q_j} d(\hat{\q}_{ij},\q_{ij})&=   -\nabla_{\q_i} d(\hat{\q}_{ij},\q_{ij})
	\end{align}
where $\hat{\q}_{ij} = \q_i \q_j^{-1}$
\end{prop}
\begin{proof}
	First recall that $d$ is bi-invariant, hence:
	\[
	d(\hat{\q}_{ij}, \q_{ij}) = d(\q_i,\q_{ij}\q_j) 
	\]
Excluding the case when  $\q_i = \q_{ij}\q_j$ (for which the expression is trivial), we have that:
	\begin{align}
		\nabla_{\q_i}d(\hat{\q}_{ij}, \q_{ij}) = \nabla_{\q_i} d(\q_i,\q_{ij}\q_j)= -\sign(\langle \q_i,\q_{ij}\q_j\rangle) \frac{\log_{\q_i}(\q_{ij}\q_j)}{\Vert \log_{\q_i}(\q_{ij}\q_j)\Vert}   
	\end{align}
	It is easy to see that $\langle \q_i,\q_{ij}\q_i \rangle= \langle \hat{\q}_{ij},\q_{ij} \rangle  $ since composition of the two rotations $\q_i$ and $\q_{ij}$ by $\q_j$ preserves the angles. On the other hand, one can observe that $\q_i^{-1}\q_{ij}q_j =  \q_j^{-1} \hat{\q}^{-1}_{ij} \q_{ij} \q_j  $ and apply \cref{eq:rotation_of_log} to get:
	\begin{align}
		\frac{\log_{\q_i}(\q_{ij}\q_j)}{\Vert \log_{\q_i}(\q_{ij}\q_j)\Vert} = \frac{\log(\q_j^{-1} \hat{\q}^{-1}_{ij} \q_{ij} \q_j)}{\Vert \log(\q_j^{-1} \hat{\q}^{-1}_{ij} \q_{ij} \q_j)\Vert} = \q_j^{-1}  \frac{\log_{\hat{\q}_{ij}}(\q_{ij})}{\Vert \log_{\hat{\q}_{ij}}(\q_{ij}) \Vert} \q_j,
	\end{align}  
	This shows the first identity. The second identity is obtained similarly. By bi-invariance of $d$, we have that $d(\hat{\q}_{ij}, \q_{ij}) = d(\q_j,\q_{ij}^{-1}\q_i)$, hence:
	\begin{align}
		\nabla_{\q_j}d(\hat{\q}_{ij}, \q_{ij}) = \nabla_{\q_j} d(\q_j,\q_{ij}^{-1}\q_i)= -\sign(\langle \q_j,\q_{ij}^{-1}\q_i\rangle) \frac{\log_{\q_j}(\q_{ij}^{-1}\q_i)}{\Vert \log_{\q_j}(\q_{ij}^{-1}\q_i)\Vert}.
	\end{align}
	Moreover, we have that $\q_j^{-1}\q_{ij}^{-1}\q_i = (\q_i^{-1}\q_{ij}\q_j)^{-1}$, thus using that $\log(\x^{-1})= -\log(\x)$ and that $\langle \q_j,\q_{ij}^{-1}\q_i \rangle = \langle \q_i,\q_{ij}\q_j \rangle$ it follows :
	\begin{align}
		\nabla_{\q_j}d(\hat{\q}_{ij}, \q_{ij}) = \sign(\langle \q_i,\q_{ij}\q_j\rangle) \frac{\log_{\q_i}(\q_{ij}\q_j)}{\Vert \log_{\q_i}(\q_{ij}\q_j)\Vert}.
	\end{align}
	which concludes the proof.
	\end{proof}

\begin{lemma}\label{eq:rotation_of_log}
	Let $x$ and $q$ be unit quaternions, then the following holds:
	\begin{align}
		\frac{\log(\q^{-1}\x\q)}{\Vert \log(\q^{-1}\x\q) \Vert} &= \q^{-1}\frac{\log(\x)}{\Vert \log(\x)\Vert}\q\\
	\end{align}
\end{lemma}
\begin{proof}
	Let's first prove the first equality, we write $\x=(b,w)$ and $\q=(c,v)$ where $b$ and $c$ are the real parts of $\x$ and $\q$ while $w$ and $v$ are their complex part.
	by definition of the quaternion product, we have that:
	\begin{align}
		\q^{-1}\x\q = (b,  (c^2- \Vert v\Vert^2 )w + 2\langle  v,w\rangle v  +2 c w\wedge v   )
	\end{align}
	Let's call $Z = (c^2- \Vert v\Vert^2 )w + 2\langle  v,w\rangle v  +2cw\wedge v $ to simplify notations. Hence, we have by definition of the logarithm:
	\begin{align}
		\frac{\log(\q^{-1}\x\q)}{\Vert \log(\q^{-1}\x \q) \Vert} = (0, \frac{Z}{\Vert Z \Vert} )
	\end{align}
	On the other hand, we also have that $\frac{\log(\x)}{\Vert\log(\x) \Vert } =  (0,\frac{w}{\Vert w\Vert})$, hence:
	\begin{align*}
		 \q^{-1}\frac{\log(\x)}{\Vert \log(\x)\Vert}\q = (0, \frac{1}{\Vert w \Vert} \left(c^2- \Vert v\Vert^2 )w + 2\langle  v,w\rangle v  +2cw\wedge v \right) := (0,\frac{Z}{\Vert w\Vert})
	\end{align*}
	We have shown that $\q^{-1}\frac{\log(\x)}{\Vert \log(\x)\Vert}\q$ and $\frac{\log(\q^{-1}\x\q)}{\Vert \log( \q^{-1}\x \q )\Vert}$ have the same direction, since both are unit vectors, they must be equal.  
\end{proof}
\subsection{Exponential map in quaternion space}\label{sec:exponential_map}
We provide a closed form expression for the exponential map used in the update rule:
let $q$ be an element in the unit quaternion manifold, i.e: $\Vert q \Vert =1$ , and $v$ an element of it's tangent space which is necessarily of the form $v=  (0,w)$ where $w$ is a vector in $\mathbb{R}^3$. Indeed, a vanishing first component insures that the $v$ doesn't contain components that are orthogonal to the unit quaternion manifold. 
\begin{align}
	exp_{q}(v) = q \exp(v) := q( cos(\Vert w \Vert), sin(\Vert w \Vert ) \frac{w}{\Vert w\Vert}   )
\end{align}

\section{The Composition Function}
\begin{itemize}
\item High entropy:
In this case $g_{ij}$ is given by:
\begin{align}
	g_{ij}(\meas) = \sum_{k_i=1}^{K_i} \sum_{k_j=1}^{K_j} w_i^{(k_i)}w_j^{(k_j)}\delta_{q_i^{(k_i)}\overline{q_j^{(k_j)}}} 
\end{align}
\item Low entropy
In this case, $K_i =K_j=K$ and the weights satisfy the additional constraint: $w_i^{(k)}=w_j^{(k)}=w^{(k)}$ for some non-negative numbers $(w^{(k)})_{1\leq k\leq K}$ that sum to $1$. Moreover, we have that:
\begin{align}
\label{eqn:lowent}
	g_{ij}(\meas) = \sum_{k=1}^{K}  w^{(k)}\delta_{q_i^{(k)}\overline{q_j^{(k)}}} 
\end{align}
\end{itemize}

\section{More Details about the Theoretical Result}
We start by detailing the assumption \textbf{H}3. In particular, we will give the precise definition of a \emph{non-degenerate} minimum. We denote by $\q$ a vector of $n$ quaternions $(\q_1,...,\q_n)$ and by $(\q^*)^{(k)}$ the $k$-th particle from the optimal distribution $\meas^*$. We will introduce the same objects as in \cite[section 3.1]{chizat2019sparse}. Note that in all our setting we fix the ratio of the learning rates $\alpha:= \frac{\eta_{\beta}}{\eta_{q}}$ to $0.1$.
Let us first define $\mathcal{J}(\meas) = \mathcal{L}(\meas) +\mathcal{R}(\meas) $ and denote by $\mathcal{J}'_{\meas^\star}$ the differential of $\mathcal{J}$ at $\meas^\star$ \cite{chizat2019sparse}.The \emph{local kernel matrix} $H$, is defined as a matrix in $\mathbb{R}^{(K*\times(1+4n))^2}$ as follows:
\begin{align}
H_{(k,l:l+4),\left(k', l':l'+4\right)}=\Biggl\{
\begin{array}{ll}{ \nabla_{\q_l, \q_{l^{\prime}}}^{2} \mathcal{J}_{\meas^{\star}}^{\prime}\left((\q^{\star})^{(k)}\right)} & {\text { if } k=k^{\prime} \text { and } l, l^{\prime} \geq 1} \\ 
{0} & {\text { if } l=0 \text { or } l^{\prime}=0}
\end{array}
 ,
\end{align}
where $\nabla^2{\q_l,\q_{l'}}$ denotes the Hessian matrix that is composed of the partial derivatives $\partial_{\q_l}\partial_{\q_{l'}}$. 
We also introduce the features $\Phi((\q_1,...,\q_n))$ and $\Psi$ defined by:
\begin{align*}
	\Phi((\q))_{ij}  &= y\mapsto k(g_{ij}(\q),  y )\\
	\Psi_{ij} &= y\mapsto \int k(y',y)\diff \meas_{ij}(y')
\end{align*}
for all $(i,j)\in\Edges$. Hence, the loss function $\mathcal{L}(\meas)$ can be re-expressed as:
\begin{align}
	\mathcal{\mathcal{L}(\meas)} = \mathcal{N}(\int \Phi(\q)\diff\meas(\q)) := \sum_{(i,j)\in\Edges} \Vert \int \phi_{ij}(\q) \diff\meas(\q) - \Psi_{ij} \Vert^{2}_{\mathcal{H}}.
\end{align}
Finally, we define the \textit{global kernel} $K$ given by:
\begin{align}
	K_{(k,l),(k',l')} = \langle \beta_{k} \bar{\nabla}\Phi((\q^{\star})^{(k)}), \beta_{k'} \bar{\nabla}\Phi((\q^{\star})^{(k')})  \rangle_{d^{2}R_{f^{\star}}}
\end{align}
where $\beta_k$ are such that the optimal weights $w_k^{\star}$ satisfy: $w_k^{\star}= \beta_k^2$, the extended gradient $\bar{\nabla}$ is defined to be $\bar{\nabla}\Phi :=  (2\alpha\Phi,\nabla \Phi )$ and the inner product is taken w.r.t. hessian of $\mathcal{N}$ at $f^{\star} = \int \Phi(\q)\diff\meas^{\star}(\q)$ as in \cite{chizat2019sparse}. Note that in general $K$ is positive semi-definite, however, as we will see now, \textbf{H}3 requires it to be definite. Now, we precise the definition of \textbf{H}3 as follows:
\begin{assumption}
The following conditions hold.
\begin{itemize}
\item The matrix $K$ is positive definite.
\item The smallest singular value of $H$ is strictly larger than $0$.
\item The only points where $\mathcal{J}' $ vanishes are the optimal particles $(\q^{\star})^{(j)}$. 
\end{itemize}
\end{assumption}

Accordingly, we precise the statement of the theorem given in the main paper.
\begin{thm}
\label{thm:global2}
Consider the LE setting \eqref{eqn:lowent} and Case 2 defined in the main paper (i.e.\ unconstrained case in~\cref{algo:flow}). Assume that \textbf{H}1-4 hold. Then, for any $0 < \varepsilon \leq 1/2$, there exists $C> 0$ and $\rho \in(0,1)$, such that  the following inequality holds:
\begin{align}
 \mathcal{J}(\meas^{(\kappa)})-\mathcal{J}(\meas^{\star}) & \leq\left(\mathcal{J}(\meas^{(0)})-\mathcal{J}(\meas^{\star})\right)\left(1-\rho\right)^{\kappa-\kappa_{0}}
\end{align}
where $\kappa = 0,1 , \dots$ denotes the iterations is a constant, and $\kappa_{0}=C /\left(\mathcal{J}(\meas^{(0)})-\mathcal{J}(\meas^{\star})\right)^{2+\epsilon}$.
\end{thm}
\section{Additional Evaluations}
\paragraph{Evaluations on the real 1D-SFM dataset~\cite{Wilson2014}}
We now evaluate two versions of our algorithm on the common benchmark introduced in 1D-SFM~\cite{Wilson2014}. In particular, we compare our results with Chatterjee and Govindu~\cite{chatterjee2017robust} as well as the Weiszfeld rotation averaging~\cite{hartley2011l1}. Our results summarized in~\cref{tab:1dsfm} demonstrate that the quality of our single particle version (\emph{Ours - K=1}) matches well with the state of the art. We also use multiple particles ($K=10$) to model the pose distributions even though we have single observed particle corresponding to the relative rotation. This is in essence similar to $K$-best synchronization~\cite{sun2019} and such approach can explain the uncertainty of the estimates as empirical distributions. The results are shown in the column \emph{Ours - K=10}. We pick the best rotation as the particle that has the maximum weight. To achieve this we use the version where we also optimize for the weights. It is seen that such a $K$-best scheme can have more chances to find the correct mode and even if we are not using explicit M-estimators yields reduced errors that are almost on par with the most-robust methods like Chatterjee \& Govindu~\cite{chatterjee2017robust}. In this evaluation we minimize the $p$-norm with $p=1.1$. In fact this finding also aligns well with our theory, where optimizing for weights allows us to find a better minimum:
A large number of particles initialized randomly ensures the coverage of all basins of attraction of the loss, while optimizing the weights allows to `kill' particles in bad local minima, in favor of those near global optima. The classical problem where $k=1$ only allows one particle which can fall in a bad local minimum and can neither escape nor be killed. This is also the reason why our theorem is not applicable to the classical synchronization.
It is noteworthy that in this evaluation we omitted the large scenes as our algorithm is computationally more costly than the algorithms specifically designed to solve the single-particle synchronization problem.
{
\begin{table}[t!]
  \centering
  \small
  \caption{Median angular errors on 1D-SFM dataset~\cite{Wilson2014}.}
    \begin{tabular}{lcccc}
    1DSFM - Scene & Chatterjee \& Govindu~\cite{chatterjee2017robust}& Weiszfeld~\cite{hartley2011l1} & Ours - K=1 & Ours - K=10 \\
    \midrule
    Alamo & 2.14  & 3.57  & 1.66  & 1.43 \\
    Ellis Island & 1.15  & 1.66  & 0.86  & 0.86 \\
    Madrid Metropolis & 3.08  & 4.37  & 4.01  & 3.50 \\
    Montreal Notre Dame & 0.71  & 0.92  & 0.92  & 0.96 \\
    NYC Library & 1.40  & 2.43  & 2.12  & 2.12 \\
    Piazza del Popolo & 2.62  & 3.35  & 2.98  & 1.26 \\
    Roman Forum & 1.70  & 2.11  & 3.61  & 5.21 \\
    Tower of London & 2.45  & 2.73  & 2.63  & 2.69 \\
    Vienna Cathedral & 4.64  & 5.14  & 1.89  & 1.70 \\
    Yorkminister & 1.62  & 2.73  & 1.89  & 1.89 \\
    \midrule
    Average & 2.15  & 2.90  & 2.26  & 2.16 \\
    \end{tabular}%
  \label{tab:1dsfm}%
\end{table}%
}

\paragraph{Measure Synchronization on the Mug Sequence}
We have now evaluated our algorithm on the \textit{mug} object shown in our main paper. To do that, we render the depth image of the 3D CAD model of the mug from various viewpoints. For each viewpoint, we back-project the depth image, creating a partial 3D view. When the handle is invisible, the partial view corresponds to a simple cylinder that is hard to match uniquely. We also store the object rotations for each view. The pose of the first image is set to identity. We then impose the graph structure by connecting each view to the $5$-nearest. For each edge, we run the voting based point pair feature matching of~\cite{drost2010model,birdal2015point}. The output of this algorithm are $K$ poses ranked by the voting score. We set $K=12$ so that the pairwise marginals contain $12$ particles. Such a procedure yields a set of diverse poses for pairs where multiple alignments are possible, and distributions with single peaks when there exist exact alignments dominating the voting table. We visualize this in~\cref{fig:suppexp}. We then run our algorithm on the obtained distributions and record the median and minimum angular errors measured by the geodesic distance. Our algorithm can match the ground truth pose by an error of $4^\circ$ while the median error of all particles is around $19^\circ$. The latter occurs as we are trying to match the entire distribution rather than a single best. Note that the ability to characterize the entirety of the possibilities is unique to our approach and as shown in this example, is of practical value. For the details of pairwise pose estimation we refer the reader to~\cite{drost2010model,birdal2015point}. Any other pairwise registration algorithm could have been used to obtain the relative rotations provided that multiple potentially uncorrelated solutions can be obtained. In this regard, ICP~\cite{besl1992method}-like algorithms such as FGR~\cite{zhou2016fast} or algorithms that strictly seek a single pose such as~\cite{mellado2014super} are discouraged.

\insertimageC{1}{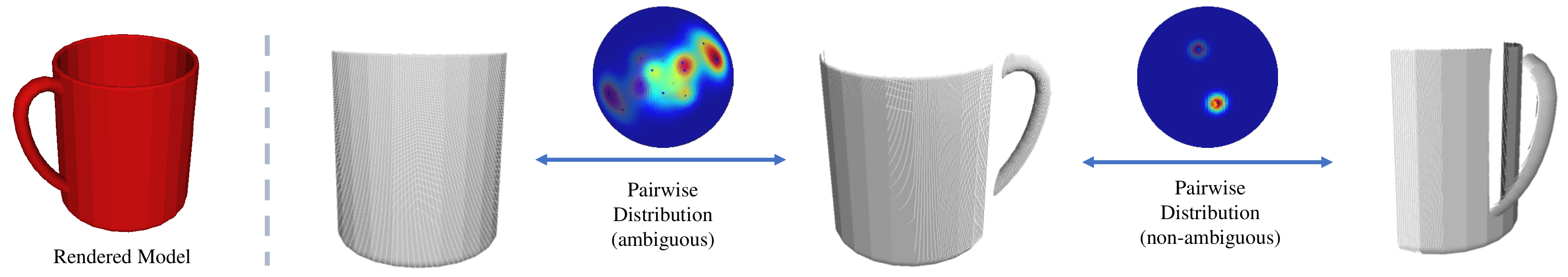}{Mug dataset. Each view is back-projected to 3D space creating a partial point cloud. We then estimate multiple possible poses for each pair of points within the vicinity by using a voting based algorithm. This figure shows that such distributions are peaked when objects can be registered uniquely and dispersed when multiple solutions do exist.
}{fig:suppexp}{h!}

\end{document}